\documentclass[twoside,11pt]{article}
\usepackage{jagi}
\usepackage{mathrsfs}
\usepackage{amsmath}

\jagiheading{11}{1}{70--85}{2020}{2020-02-16}{2020-09-29}{10.2478/jagi-2020-0004}
{S.\ Alexander}

\ShortHeadings{The Archimedean trap}{S.\ Alexander}

\author{\name Samuel Allen Alexander
\email samuelallenalexander@gmail.com \\
\addr Quantitative Research Analyst\\
The U.S.\ Securities and Exchange Commission\\
New York Regional Office
}

\editor{Alexey Potapov}
\firstpageno{70}

\title{The Archimedean trap: Why
traditional reinforcement learning will probably not yield AGI}

\begin{document}

\maketitle

\begin{abstract}
    After generalizing the Archimedean property of real numbers in such a
    way as to make it adaptable to non-numeric structures, we demonstrate
    that the real numbers cannot be used to accurately measure non-Archimedean
    structures. We argue that, since an agent with Artificial General
    Intelligence (AGI) should have no problem engaging in tasks that inherently
    involve non-Archimedean rewards, and since traditional reinforcement
    learning rewards are real numbers, therefore traditional reinforcement
    learning probably will not lead to AGI. We indicate two possible ways traditional
    reinforcement learning could be altered to remove this roadblock.
\end{abstract}

\section{Introduction}

Whenever we measure anything using a particular number system, the
corresponding measurements will be constrained by the structure of that
number system. If the number system has a different structure than
the things we are measuring with it, then our
measurements will suffer accordingly, just as if we were trying to
force square pegs into round holes.

For example, the natural numbers make lousy candidates for measuring
lengths in a physics laboratory. Lengths in the lab have
properties such as, for example, the fact that for any two distinct
lengths, there is an intermediate length strictly between them.
The natural numbers lack this property. Imagine the poor physicist,
brought up in a world of only natural numbers, scratching his or her
head upon encountering a rod with length strictly between two rods
of length $1$ and $2$.

It is tempting to think of the real numbers $\mathbb R$---i.e., the unique
complete ordered field---as a generic number system with whatever
structure suits our needs. But the
real numbers do have their own specific structure. That structure is
flexible enough to accomodate many needs, but we shouldn't just
take that for granted. One particular property constraining the real numbers
is the following.

\begin{lemma}
\label{specializedarchimedeanlemma}
(The Archimedean Property\footnote{The Archimedean property is named after Archimedes of Syracuse.
A similar property appears as the fifth axiom in his \emph{On the Sphere
and Cylinder} \citep{archimedes}:
\begin{quote}
    Further, of unequal lines, unequal surfaces, and unequal
    solids, the greater exceeds the less by such a magnitude
    as, when added to itself, can be made to exceed any
    assigned magnitude among those which are comparable with
    [it and with] one another.
\end{quote}
Note that Archimedes specifically restricts his statement to
lengths, surface areas and volumes, in fact going out of his
way to limit the magnitudes to which said length/area/volume
can be made to exceed (he could have saved some words
by stopping his sentence at ``...can be made to exceed any
assigned magnitude'', if that were his intention).

The Archimedean property is also closely related to Definition 4 of Book V of Euclid's
\emph{Elements} \citep{euclid}:
\begin{quote}
    (Those) magnitudes are said to have a ratio
    with respect to one another which, being
    multiplied, are capable of exceeding one
    another.
\end{quote}
Proposition 1 of Book X is also relevant.
Many math historians
speak of the modern-day Archimedean property, Archimedes' 5th axiom, and Euclid's
properties as being identical, but in fact they are all subtly different from one
another, see \cite{bair2013mathematical}.})
Let $r>0$ be any positive real number.
For every real number $y$, there is some natural number $n$
such that $nr>y$.
\end{lemma}

Rather than directly prove Lemma \ref{specializedarchimedeanlemma},
we will prove a generalized result which, we will argue, is
more adaptable to other structures.

\begin{lemma}
\label{generalizedarchimedeanlemma}
(The Generalized Archimedean Property)
Let $r>0$ be any positive real number.
For any $x,y\in\mathbb R$, say that $x$ is \emph{significantly less}
than $y$ if $x\leq y-r$.
If $x_0,x_1,x_2,\ldots$ is any infinite sequence of real numbers,
where each $x_i$ is significantly less than $x_{i+1}$, then for every real number $y$,
there exists some $i$ such that $y$ is significantly less than $x_i$.
\end{lemma}

\begin{proof}
If not, then there is some $y$ such that $y+r > x_i$ for all $i$.
Thus, $X=\{x_0,x_1,x_2,\ldots\}$ has an upper bound. By the completeness
of $\mathbb R$, $X$ must have a least upper bound $z\in\mathbb R$.
Since $z$ is the least upper bound for $X$, $z-r$ is not an upper bound
for $X$, so there is some $i$ such that $x_i>z-r$.
By assumption, $x_i\leq x_{i+1}-r$, so $x_{i+1}>z$, contradicting the choice
of $z$.
\end{proof}

Lemma \ref{specializedarchimedeanlemma} follows from
Lemma \ref{generalizedarchimedeanlemma} by letting $x_i=ir$.

The above property is automatically inherited by subsystems
of the reals, such as the rational numbers $\mathbb Q$, the natural
numbers $\mathbb N$, the integers $\mathbb Z$, or the algebraic numbers.
All inherit the Generalized Archimedean Property in obvious ways.

Lemma \ref{generalizedarchimedeanlemma} allows us to adapt the notion
of Archimedeanness to other things than real numbers, even to things
for which there is no notion of arithmetic at all
(Lemma \ref{specializedarchimedeanlemma} would not adapt to such things).
All we need is a notion of ``significantly less than''.
For any set of things, some of which are ``significantly less than''
others, we can ask whether or not the property in Lemma
\ref{generalizedarchimedeanlemma} holds. We will make this formal in
Section \ref{backgroundsection}.

\begin{example}
\label{fuzzywidgets}
(Fuzzy widgets)
Suppose we have some fuzzy widgets, and we observe that certain
widgets are fuzzier than others. Naturally, we are inclined to
quantify the fuzziness of the widgets, assigning them numerical
fuzziness measures from some number system. Nine times out of ten,
we choose to use the real numbers, or a subsystem thereof,
often without a second thought. But suppose
among these widgets, there happen to be widgets $x_0,x_1,\ldots$
such that each $x_{i}$ is significantly less fuzzy than $x_{i+1}$,
and another widget $y$ such that $x_i$ is
significantly less fuzzy than $y$ for $i=0,1,\ldots$.
Suddenly, our decision to use real numbers puts us in a bind.
It is impossible to assign real number fuzziness measures to our
widgets in such a way that significantly less fuzzy widgets get
significantly smaller real number measures. That would
contradict Lemma \ref{generalizedarchimedeanlemma}.
\end{example}

Note that the above example does not require us to have any notion
of multiplying fuzziness by a natural number $n$ (as we would need to have
if we wanted to adapt Lemma \ref{specializedarchimedeanlemma}).
This illustrates the enhanced adaptability of Lemma \ref{generalizedarchimedeanlemma}.

The structure of this paper is as follows.
\begin{itemize}
    \item
    In Section \ref{backgroundsection} we formally adapt
    Lemma \ref{generalizedarchimedeanlemma} to obtain a notion of Archimedeanness
    for non-numerical structures, and demonstrate that non-Archimedean such
    structures cannot accurately be measured using the real numbers.
    \item
    In Section \ref{reinforcementlearningsection} we argue that
    traditional reinforcement learning probably will not lead to AGI because
    its rewards are overly constrained.
    \item
    In Section \ref{nontraditionalsection} we discuss non-traditional
    variations of reinforcement learning that avoid the problem of
    overly constrained rewards.
    \item
    In Section \ref{conclusionsection} we summarize and make concluding remarks.
\end{itemize}

\section{Generalized Archimedean Structures}
\label{backgroundsection}

The real numbers possess the Archimedean property, but other structures
may or may not. To make this more precise,
we introduce the following formalism, adapting from Lemma \ref{generalizedarchimedeanlemma}.

\begin{definition}
\label{archimedeandefn}
    A \emph{significantly-ordered structure} is a collection $X$ with
    an ordering $\ll$.
    For $x_1,x_2\in X$, we say $x_1$ is \emph{significantly less than}
    $x_2$ if $x_1\ll x_2$.
    A significantly-ordered structure is \emph{Archimedean} if it
    has the following property: for every $X$-sequence
    $x_0\ll x_1\ll x_2 \ll \cdots$,
    for every $y\in X$, there is some $i\in\{0,1,\ldots\}$ such that $y\ll x_i$.
\end{definition}

For any real number $r>0$, a prototypical example of an Archimedean
significantly-ordered structure is the real
numbers with $\ll$ defined such that
$x_1\ll x_2$ if and only if $x_1\leq x_2-r$.

\begin{definition}
    Suppose $(X,\ll)$ is a significantly-ordered structure.
    A function $f:X\to\mathbb R$ is said to \emph{accurately measure $(X,\ll)$}
    if there is some real $r>0$ such that the following requirement holds:
    \begin{itemize}
        \item
        For all $x_1,x_2\in X$, $x_1\ll x_2$ if and only if
        $f(x_1)\leq f(x_2)-r$.
    \end{itemize}
\end{definition}

The following proposition formalizes the dilemma we illustrated in
Example \ref{fuzzywidgets}.

\begin{proposition}
\label{maindilemma}
(Inadequacy of the reals for non-Archimedean structures)
    Suppose $(X,\ll)$ is a significantly-ordered structure.
    If $X$ is non-Archimedean, then no function $f:X\to\mathbb R$
    accurately measures $(X,\ll)$.
\end{proposition}

\begin{proof}
    Assume, for sake of a contradiction, that some $f:X\to\mathbb R$
    exists which accurately measures $(X,\ll)$. Thus there is some real $r>0$ such that
    for all $x_1,x_2\in X$, $x_1\ll x_2$ if and only if $f(x_1)\leq f(x_2)-r$.
    Since $X$ is non-Archimedean, there is some $X$-sequence
    $x_0\ll x_1\ll x_2\ll\cdots$ and some $y\in X$
    such that there is no $i$ such that $y\ll x_i$.
    By choice of $r$, each $f(x_{i})\leq f(x_{i+1})-r$ and there is no
    $i$ such that $f(y)\leq f(x_i)-r$.
    This contradicts Lemma \ref{generalizedarchimedeanlemma}.
\end{proof}

Proposition \ref{maindilemma} tells us that we cannot accurately measure
non-Archimedean structures using real numbers\footnote{There is an area of research
known as \emph{measurement theory}, which, traditionally,
``takes the real numbers as a pre-given numerical domain'' \citep{niederee1992numbers}.
Some work has been done to generalize measurement theory away from this assumption
\citep{narens1974, skala1975, rizza2016divergent}.
We would submit this paper as further motivation in that direction.}.
Any attempt to do so will necessarily be misleading, because ordering
relationships among the non-Archimedean structures will fail to be reflected
by the real-number measurements given to them.
We will inevitably end up like the puzzled physicist
brought up in a world of only natural numbers, confronted by a rod of
length $1.5$.

\begin{remark}
(Spearman's Law of Diminishing Returns)
Suppose $(X,\ll)$ is a non-Archimedean significantly-ordered structure
with elements $x_0,x_1,\ldots$ and $y$ such that
$x_0\ll x_1\ll\cdots$ and each $x_i\ll y$.
Suppose $f:X\to\mathbb R$ has the property that $f(x_0)<f(x_1)<\cdots$
and each $f(x_i)<f(y)$. Then the monotone convergence theorem implies
that $\lim_{i\to\infty}f(x_i)$ converges to some finite value.
This suggests a general \emph{law
of diminishing returns}: any time a non-Archimedean significantly-ordered
structure $(X,\ll)$ is measured using
real numbers, if the measurement does not blatantly violate $\ll$ (in other
words, if there are no $x_1\ll x_2$ such that $x_1$ is given a larger real-number
measurement than $x_2$), then there will inevitably be elements $x_0\ll x_1\ll \cdots$
exhibiting \emph{diminishing returns}, in the sense that the measurements of
$x_i$ and $x_j$ are approximately equal for all large enough $i,j$.
If human intelligence is non-Archimedean, this could potentially shed light on
a psychometrical phenomenon called \emph{Spearman's Law of Diminishing Returns}
\citep{spearman1927abilities, blum2017spearman, hernandez2019ai}, the empirical
tendency of cognitive ability tests to be less correlated in
high-intelligence populations. Even tiny measurement errors would eventually
dominate the test result differences as the true measurements plateau.
\end{remark}

\begin{example}
\label{nonexamples}
(Examples of non-Archimedean structures)
    \begin{itemize}
        \item
        (Sets)
        Say that set $x_1$ is significantly smaller than set $x_2$ if there is an injective
        function from $x_1$ into $x_2$ but there is no bijective function from $x_1$
        onto $x_2$. It is easy to show there are sets $x_0,x_1,\ldots$, with each
        $x_{i}$ significantly smaller than $x_{i+1}$, and each $x_i$ is
        significantly smaller than
        $y=\cup_{i=0}^\infty x_i$.
        Thus, sets are
        non-Archimedean. In the field of \emph{set theory}, mathematicians measure
        the size of sets using Georg Cantor's famous non-Archimedean number system,
        the cardinal numbers.
        \item
        (Logical theories)
        It is not difficult to come up with (for example) true
        theories $x_0,x_1,\ldots$ (in the language of arithmetic) such that
        each $x_{i+1}$ proves the consistency of $x_i$, and an additional
        true theory $y$ (in the language of arithmetic)
        which proves the consistency of $\cup_{i=0}^\infty x_i$.
        In a sense, then, each $x_i$ is significantly weaker than $x_{i+1}$
        (see G\"odel's incompleteness theorems), and each $x_i$ is
        significantly weaker than $y$. In this sense, logical theories
        are non-Archimedean. In the field
        of \emph{proof theory} \citep{pohlers2008proof, rathjen},
        logicians measure the logical strength of theories using computable
        ordinal numbers, another non-Archimedean number system.
        \item
        (Asymptotic runtime complexities)
        Suppose $x_0,x_1,\ldots$ are algorithms such that each $x_i$ has
        runtime complexity $\Theta(n^i)$, and suppose $y$ is an algorithm
        with runtime complexity $\Theta(2^n)$. Then in a certain sense, each
        $x_{i}$ has significantly lower asymptotic runtime complexity than $x_{i+1}$,
        and each $x_i$ has significantly lower asymptotic runtime complexity than
        $y$. In this sense, asymptotic runtime complexity is non-Archimedean.
        In computer science, these runtime complexities are usually measured using
        big-$O$, big-$\Theta$, or similar notation systems.
    \end{itemize}
\end{example}

\begin{example}
\label{speculativeexamples}
    (Speculative examples of potentially non-Archimedean structures)
    Certain structures might plausibly be non-Archimedean, but it is a difficult
    question to say whether they truly are or not. The reader could come up with
    such examples in great abundance.
    \begin{itemize}
        \item
        (Musical beauty)
        Assuming there is such a thing as objective
        musical beauty (not contingent on features of the human condition, etc.),
        then it is plausible that musical beauty might be non-Archimedean, in the following
        sense: there might be songs $x_0,x_1,\ldots$ such that each $x_i$
        is significantly less beautiful than $x_{i+1}$, and another song
        $y$ such that each $x_i$ is significantly less beautiful than $y$.
        \item
        (Ethical utility)
        Early utilitarian Jeremy Bentham suggested a hedonistic
        calculus in which pleasure measurements would be assigned to
        actions, to help adjudicate ethical dilemmas.
        His successor, John Stuart Mill, objected that some actions are incomparably
        better than others: ``If one of [two pleasures] is, by those
        who are competently acquainted with both, placed so far above the other that
        they prefer it ...\ and would not resign it for any quantity of the other
        pleasure which their nature is capable of, we are justified in ascribing to
        the preferred enjoyment a superiority in quality, so far outweighing quantity
        as to render it, in comparison, of small account'' \citep{mill}.
        This suggests that Bentham's pleasures are non-Archimedean.
        \item
        (AGI)
        It is plausible that there
        are\footnote{As hinted by Protagoras, assuming Protagoras's own intelligence
        stays constant and remains higher than the intelligence of his student
        and that they live forever and that \emph{better} means \emph{significantly
        better}:
        ``The very day you start, you will go home a better man, and the same thing
        will happen the day after. Every day, day after day, you will get better
        and better'' \citep{protagoras}.} AGIs $x_0,x_1,\ldots$ such that
        each $x_{i}$ is significantly less
        intelligent than $x_{i+1}$, and another AGI $y$ such that each $x_i$
        is significantly less
        intelligent than $y$. We first pointed this out in
        \citep{alexander2019measuring}, where we propose measuring the
        intelligence of mechanical
        knowing agents using computable ordinals, the same non-Archimedean number system
        which proof theorists use to measure logical strength of mathematical
        theories. Incidentally, if AGI intelligence is non-Archimedean, then
        Proposition \ref{maindilemma} shows it is
        impossible to measure machine intelligence using real numbers without some
        of those measurements being misleading\footnote{This would solve an open problem
        implicitly stated by \citet{legg} when they said of their
        real-number universal intelligence measure: ``...none of these people have
        been able to communicate why the work [on measuring universal intelligence
        using real numbers] is so obviously flawed in any concrete way ...
        If anyone would like to properly explain their position to us in the future,
        we promise not to chase you down the street!''}.
        \item
        (Nonstandard cosmologies)
        Some
        authors
        \citep{al2016surreal, andreka2012logic, reeder2012infinitesimals,
        rosinger2007cosmic, chen2019infinitesimal} have
        even speculated about the nature of non-Archimedean space and/or
        time.
    \end{itemize}
\end{example}

\section{Reinforcement learning}
\label{reinforcementlearningsection}

In reinforcement learning (RL), an agent interacts with an environment,
taking actions from a fixed set of possible actions. With every action the
agent takes, the environment responds with a new observation and with a
reward. In traditional RL, these rewards are real numbers (many
authors further constrain them to be rational numbers).

By restricting rewards to be real (or rational) numbers, we unconsciously
constrain RL to only be applicable toward tasks of an inherently Archimedean
nature. For example, \citet{wirth2017survey} point out that
in tasks related to cancer treatment \citep{zhao2009reinforcement},
``the death of a patient should be avoided at any cost. However, an
infinitely negative reward breaks classic reinforcement learning algorithms
and arbitrary, finite values have to be selected.'' This problem could be
avoided if instead of real numbers, rewards were drawn from a suitable
non-Archimedean number system containing negative infinities. Doing so would
be a departure from traditional RL.

\begin{example}
\label{musicalexample}
To give an intuitive example, assume that musical beauty is non-Archimedean,
as in Example \ref{speculativeexamples}. We can imagine environments where the
RL agent is tasked with composing songs. For example, the possible actions the
agent is allowed to take might include one action for each piano key, plus
an additional ``stand and bow'' action to signal that a song is
finished.
Whenever the agent stands and bows, the agent is rewarded with applause based on
the beauty of the song the agent
composed\footnote{To quote \cite{wang2015assumptions}: ``Decision makings often do not happen
at the level of basic operations, but at the level of composed actions, where
there are usually infinite possibilities.''}. Assuming
musical beauty is
non-Archimedean, such an environment falls outside the possibility of traditional
RL. By Lemma \ref{maindilemma}, there is no way to assign real number
rewards to songs without misleading the agent. If $x_0,x_1,x_2,\ldots$ are songs
where each $x_{i}$ is significantly less beautiful than $x_{i+1}$, and
all the $x_i$ are significantly less beautiful than another song $y$,
then there is no way to
assign real-valued rewards to these songs such that each $x_{i}$ gets
significantly less reward than $x_{i+1}$ and significantly less reward
than $y$.
\end{example}

Or, to re-use the cancer example, assume there are certain bad procedures the
robotic surgeon could take, each one significantly worse than the previous,
but all still significantly better than killing the patient. There is no way
to assign real-valued rewards to these actions, and to killing the patient,
in such a way that each bad action gets punished significantly harsher than
the previous, but still significantly more forgivingly than the punishment for
killing the patient.

The reader might object by challenging the non-Archimedeanness
of music and of medical procedures. But we only used those to make the examples
more intuitive. If the reader insists, we can resort to mathematical tasks.

\begin{example}
\label{theoryexample}
Imagine that the agent is tasked with
typing up mathematical theories, and when the agent stands and bows, the agent
is rewarded with applause based on the proof-theoretical strength of the theory
(or hit with tomatoes if the theory is inconsistent).
In Example \ref{nonexamples} we noted that proof-theoretical strength of theories
is non-Archimedean. There exist theories $x_0,x_1,\ldots$, each significantly
proof-theoretically weaker than the next, and another theory $y$,
significantly proof-theoretically stronger than all the $x_i$'s. We cannot
possibly assign real-valued rewards to these theories without misleading the
agent.
\end{example}

The reader might object to Example \ref{theoryexample} on the grounds that
judging the proof-theoretical
strength of a theory is inherently non-computable anyway. The example could be
modified so that instead of typing up mathematical theories, the agent has to
type up mathematical subtheories in (say) the language of Peano arithmetic,
accompanied by consistency proofs in (say) ZFC. It can be shown that the
proof-theoretical strength of mathematical theories is still non-Archimedean,
even when restricted to subtheories of arithmetic whose consistency can be
proven in ZFC\footnote{For example, let $x_0$ be the theory of Peano arithmetic,
and for each $i$, let $x_{i+1}$ be $x_i$ together with $CON(x_i)$, a canonical axiom
encoding the consistency of $x_i$. Let $y$ be the theory of Peano arithmetic
along with $CON(\cup_i x_i)$. ZFC is certainly adequate to prove the consistency of
each of these theories. In the sense of Example \ref{nonexamples}, each $x_{i}$ is
significantly weaker than $x_{i+1}$ and significantly weaker than $y$.}.

The reader might object that the above theories-with-proofs example is contrived.
But an AGI with human or better intelligence should have no problem
at least comprehending and attempting such a task (regardless of whether or not
the AGI is able to perform well at it). When we prove that the Halting Problem
is unsolvable, we do so by considering contrived programs that we could write if
the Halting Problem were solvable. The contrivedness of those programs does not
invalidate the proof of the unsolvability of the Halting Problem. Again, when we
prove that C++ templates are Turing complete \citep{veldhuizen}, we do so by
considering extremely
bizarre C++ templates that would never arise naturally in a software
studio. This does not invalidate the proof that C++ templates are Turing complete.

Finally, the reader might object that approximating infinite rewards with arbitrary
large finite rewards is good enough. Who cares (the argument might go) whether
pushing a button gives the agent infinite pleasure or only a million units of pleasure?
Either way (the argument goes) the agent is going to learn to prefer that button
over a button that gives only $.1$ units of pleasure. The following example shows
that this logic breaks down in non-Markov environments.

\begin{example}
\label{funnyexample}
(Delayed gratification)
Consider an environment with a red button and a blue button.
Pushing the red button always grants $+1$ reward.
As for the blue button, suppose the agent presses the blue button for the
$i$th time. If $i=2^j$ for some integer $j$, then the agent shall receive a
reward of $\omega$ (the smallest infinite ordinal), but otherwise, the agent
shall receive $0$ reward.
If we approximate $\omega$ with a real-value of, say, $1,000,000$,
then after a long enough time spent in the environment, an AGI will be misled into
thinking that it isn't worth the longer and longer wait-times between blue-button
rewards: eventually, it will take more than $1,000,000$ blue-button-presses to get
rewarded, and the rewards will mislead the AGI into thinking
it more worthwhile to get the guaranteed $+1$ reward from the red button.
\end{example}

Our critic could respond to Example \ref{funnyexample}
by making the approximation dynamic, say, making the
$2^j$th press of the blue button grant $1000000\cdot 2^j$ reward, but at this point,
the critic is clearly just hard-coding the correct actions into the reward function,
something which is only possible in Example \ref{funnyexample} because the environment
is simple enough that we can completely understand it ourselves. For the
kinds of non-trivial environments where AGI would actually be useful, such carefully
engineered reward approximations would quickly become intractible.

Reinforcement Learning is useful for many practical tasks, but at least in
its traditional flavor, it is too constrained (by its arbitrary choice of number
system for its rewards) to apply to certain
non-Archimedean tasks\footnote{Perhaps explaining why
``despite almost two decades of RL research, there has been little solid
evidence of RL systems that may one day lead to [AGI]''
\citep{livingston}.}, which, however contrived they are, could certainly be
attempted by an AGI. Traditional reinforcement learning will probably
not lead to AGI.

\section{Non-traditional reinforcement learning}
\label{nontraditionalsection}

We have argued that traditional RL will probably not lead to AGI, because
an AGI is capable of attempting non-Archimedean tasks whose rewards are
too rich to express using real numbers. There are at least two
potential ways to change RL so as to make it applicable to such tasks and,
thus, at least potentially capable of leading to AGI. Of course, there is
no guarantee that removing the roadblock in this paper will cause RL to
lead to AGI. There might be other roadblocks besides the inadequate reward
number system\footnote{For example, many RL authors consider non-deterministic
environments where rewards and observations include an element of
randomness. The probabilities involved are, traditionally, assumed to be
real numbers. Perhaps some recent work \citep{benci2013non} on non-Archimedean
probability could be relevant against that roadblock.}.

\subsection{Preference-based reinforcement learning}

A lot of exciting research has been done on non-traditional variations
of RL where, instead of giving the agent numerical rewards for taking actions,
one instead informs the agent about the relative preference of various
actions or action-sequences. See \citep{wirth2017survey} for a survey.
This nicely side-steps the problems from this paper.


\subsection{Reinforcement learning with other number systems}

The most obvious way to modify RL to avoid the problems presented in this
paper is to change which number system is used\footnote{Anticipated
by \citet{rizza2016divergent}.}.
As far as this author is aware,
the choice to use real (or rational) numbers for rewards was not made based
on any fundamental criteria\footnote{\cite{niederee1992numbers} points out
that there are
no deeper reasons to assume that the number system should necessarily
even have the same cardinality as $\mathbb R$. And \cite{rizza2016divergent} says:
``No particular feature of the space of informational
states suggests that such a codomain [as $\mathbb R$] should be selected''.}. The real
(or rational) numbers are currently a
useful pragmatic choice because they are easy to compute with using 21st
century software and 21st century school curricula, but that's hardly relevant
in the field of genuine AGI. One might
say the real numbers were a good choice because they are familiar, but even
that is arguable: in general, students are usually not taught what the
real numbers \emph{actually are},
unless they major in pure mathematics at the university level. Anyway,
the familiarity argument is totally irrelevant in the field of AGI.

Various non-Archimedean number systems exist. Number systems can be
discrete or continuous; the nature of reinforcement learning clearly
suggests a continuous number system. We will consider three continuous
number systems: formal Laurent series; hyperreal numbers; and
surreal numbers.

\subsubsection{Formal Laurent series}

David \cite{tall1980looking} described the following
real-number-extending number system (which he called the ``superreals'', but
that vocabulary does not seem to have caught on).

\begin{definition}
A \emph{formal Laurent series} is a formal expression of the form
\[
\sum_{j=-\infty}^\infty a_j \epsilon^j
\]
such that there is some integer $j_0$ such that $a_j=0$ for all $j<j_0$.
Suppose $A=\sum_{j=-\infty}^\infty a_j\epsilon^j$
and $B=\sum_{j=-\infty}^\infty b_j\epsilon^j$ are two distinct formal Laurent series.
We declare $A<B$ if and only if $a_j<b_j$ where $j$ is the smallest
index such that $a_j\not=b_j$.
\end{definition}

For brevity, we will write $a\epsilon^b$ for the formal Laurent series
$\sum_{j=-\infty}^{\infty} a_j\epsilon^j$ where $a_b=a$ and $a_j=0$ for all $j\not=0$.
Likewise, we may write, for example, $5\epsilon^{-1}+2\epsilon^3$ for the
formal Laurent series $\sum_{j=-\infty}^{\infty} a_j\epsilon^j$ where $a_{-1}=5$,
$a_3=2$, and $a_j=0$ for all $j\not\in\{-1,3\}$, and similarly for other
finite sums of powers of $\epsilon$.

We can consider the real numbers $\mathbb R$ to be embedded in the formal Laurent
series by way of the embedding $r\mapsto r\epsilon^0$. Having done so, the intuition is that,
for example, $1\epsilon^1$ is what we might call a ``first-order
infinitesimal number'', smaller
than every positive real;
$1\epsilon^2$ is what we might call a ``second-order infinitesimal number'',
smaller than every
positive first-order infinitesimal number; and so on. Likewise,
$1\epsilon^{-1}$ is what we might call a ``first-order infinite number'', bigger than
every real; $1\epsilon^{-2}$ is what we might call a ``second-order infinite number'', bigger
than every first-order infinite number; and so on. Thus, the formal Laurent series
should suffice to address the specific problem described by
\cite{wirth2017survey} in which an infinite negative reward is required when the
RL agent kills the cancer patient.

There are natural ways to define arithmetic on formal Laurent series, but we will
avoid those details here.
The advantage of the formal Laurent series number system is that it is
relatively concrete, compared to the more abstract hyperreal or surreal numbers
discussed below.

\begin{example}
(Examples of formal Laurent series comparisons)
\begin{enumerate}
    \item Consider $A=5\epsilon^{-1}-2\epsilon^{0}+3\epsilon^1+4\epsilon^2$ and
    $B=5\epsilon^{-1}-2\epsilon^0+1\epsilon^1+4\epsilon^2+5\epsilon^6$.
    The $\epsilon^{-1}$- and $\epsilon^0$-coefficients of $A$ and $B$ are equal, so
    we compare their $\epsilon^1$-coefficients. $A$ has an $\epsilon^1$-coefficient of $3$ and
    $B$ has an $\epsilon^1$-coefficient of $1$, and $3>1$, so $A>B$.
    \item Consider $A=999999\epsilon^5$ and $B=0.00001\epsilon^{4}$.
    The $\epsilon^{4}$-coefficient of $A$ is $0$, which is smaller than
    $B$'s $\epsilon^{4}$-coefficient ($0.00001$), so $A<B$.
\end{enumerate}
\end{example}

There is a natural way to consider formal Laurent series as a significantly-ordered
structure, generalizing the notion of ``significantly greater than'' from
Lemma \ref{generalizedarchimedeanlemma}.

\begin{definition}
\label{significantorderednessoflaurent}
\begin{enumerate}
\item
For every Laurent series $A=\sum_{j=-\infty}^\infty a_j\epsilon^j$,
let $o(A)$ (the \emph{order} of $A$)
be the smallest integer $j$ such that $a_j\not=0$, or $o(A)=\infty$ if
there is no such $j$.
Let $LC(A)$ (the \emph{leading coefficient} of $A$)
be the $\epsilon^{o(A)}$-coefficient of $A$, or $0$ if $o(A)=\infty$.
\item
Let $r>0$ be any positive real number.
For any formal Laurent series $A=\sum_{j=-\infty}^\infty a_j\epsilon^j$
and $B=\sum_{j=-\infty}^\infty b_j\epsilon^j$,
we say $A\ll_r B$ if one of the following conditions holds:
    \begin{itemize}
        \item
        $o(A)>o(B)$ and $LC(B)>0$; or
        \item
        $o(A)<o(B)$ and $LC(A)<0$; or
        \item
        $o(A)=o(B)$ and $LC(A)\leq LC(B)-r$.
    \end{itemize}
\end{enumerate}
\end{definition}

\begin{lemma}
For any real $r>0$, the formal Laurent series, considered as a significantly-ordered
structure according to $\ll_r$, are non-Archimedean.
\end{lemma}

\begin{proof}
    Let $r>0$.
    Recall from Definition \ref{archimedeandefn} that the formal Laurent series,
    considered as a significantly-ordered structure according to $\ll_r$, are
    Archimedean if and only if the following statement is true:
    \begin{itemize}
    \item
        For every sequence $x_0\ll_r x_1\ll_r x_2 \ll_r \cdots$ of formal Laurent series,
        for every formal Laurent series $y$, there is some $i$ such that
        $y\ll_r x_i$.
    \end{itemize}
    We will exhibit a particular sequence $x_0\ll_r x_1\ll_r x_2\ll_r\cdots$
    of formal Laurent series, and a formal Laurent series $y$, such that the
    above statement fails, thereby showing that the formal Laurent series
    are non-Archimedean.

    Let $x_0=r\epsilon^1$, $x_1=2r\epsilon^1$, $x_2=3r\epsilon^1$, and in general let
    $x_i=(i+1)r\epsilon^1$. Let $y=1\epsilon^0$.
    Thus $o(x_0)=o(x_1)=o(x_2)=\cdots=1$,
    and $LC(x_0)=r$, $LC(x_1)=2r$, $LC(x_2)=3r$, and in general $LC(x_i)=(i+1)r$
    for each $i=0,1,2,\ldots$. Meanwhile, $o(y)=0$ and $LC(y)=1$.

    By Definition \ref{significantorderednessoflaurent},
    each $x_i\ll_r x_{i+1}$ for $i=0,1,2,\ldots$ (because each $o(x_i)=o(x_{i+1})=1$
    and each $LC(x_i)=(i+1)r\leq LC(x_{i+1})-r=(i+2)r-r=(i+1)r$).
    Thus, if the formal Laurent series were Archimedean, there would have to be some
    $i\in\{0,1,\ldots\}$
    such that $y\ll_r x_i$. This is impossible because
    for any such $i$, $o(y)=0<1=o(x_i)$ and $LC(x_i)=(i+1)r>0$.
\end{proof}

Unfortunately, although the formal Laurent series contain infinities and infinitesimals,
in a sense we will make formal, they still do not contain ``enough'' infinities and
infinitesimals to accomodate the fully general environments that an AGI should be able
to navigate. To make this formal, we introduce a weaker
notion of Archimedeanness.

\begin{definition}
Suppose $(X,\ll)$ is a significantly-ordered structure.
We define a new order $\ll'$ on $X$ as follows.
For any $x,y\in X$, we declare $x\ll' y$ if and only if there is an $X$-sequence
$x_0,x_1,\ldots$ such that the following conditions hold:
\begin{enumerate}
    \item
    $x_0=x$.
    \item
    Each $x_i\ll x_{i+1}$.
    \item
    Each $x_i\ll y$.
\end{enumerate}
We say $(X,\ll)$ is \emph{semi-Archimedean} if
the following condition holds:
\begin{itemize}
    \item
    For every $X$-sequence $x_0\ll' x_1\ll' x_2\ll' \cdots$,
    for every $y\in X$, there is some $i$
    such that $y\ll x_i$.
\end{itemize}
\end{definition}

To be semi-Archimedean is a weaker condition than to be Archimedean, but it is
still a condition, and one which there is evidently no reason to assume should
constrain reinforcement learning rewards in general.
For example, just as it is unclear whether musical beauty is Archimedean,
likewise, it is unclear whether musical beauty is even semi-Archimedean.
If musical beauty is not semi-Archimedean,
then Example \ref{musicalexample} does not merely suggest the
inadequacy of real number rewards, but of rewards from any semi-Archimedean number
system. And if musical beauty is too informal,
we can still fall back to mathematical theories (Example \ref{theoryexample}),
for the strength of mathematical theories can be shown not to be semi-Archimedean.

In the following theorem, we show that the formal Laurent series are semi-Archimedean.
By the above paragraph, this suggests that even if we extended reinforcement
learning to permit
formal Laurent series rewards, the resulting framework would still probably not lead
to AGI.

\begin{theorem}
\label{tedioustheorem}
For any real $r>0$,
the formal Laurent series are semi-Archimedean when considered as a
significantly-ordered structure as in Definition \ref{significantorderednessoflaurent}.
\end{theorem}

\begin{proof}
Let $r>0$. For simplicity, we write $\ll$ for $\ll_r$ and $\ll'$ for
$\ll'_r$.

Claim 1: Whenever $a\ll' b$, then $b\geq 0$.
To see this, assume $a\ll' b$, so there are $a=x_0,x_1,\ldots$ such that each
$x_i\ll x_{i+1}$ and each $x_i\ll b$.
If any $x_i\geq 0$ then, since $x_i\ll b$, it follows that $b\geq 0$, as desired.
But suppose every $x_i<0$.
If all the $o(x_i)$ were equal, then, since each
$x_i\ll x_{i+1}$, it would follow that each $LC(x_{i})\leq LC(x_{i+1})-r$, so Lemma
\ref{generalizedarchimedeanlemma} would imply $LC(x_i)>0$ for some $i$, contradicting
the assumption that $x_i<0$. So there is some minimal $i_1$ such
that $o(x_{i_1+1})\not=o(x_{i_1})$,
and since $x_{i_1}$ and $x_{i_1+1}$ are negative, this implies
$o(x_{i_1+1})<o(x_{i_1})$, and by minimality of $i_1$, $o(x_j)=o(x_0)$ for all
$j\leq i_1$.
By identical reasoning applied to the sequence $x_{i_1},x_{i_1+1},\ldots$,
there is some minimal $i_2>i_1$ such that $o(x_{i_2+1})<o(x_{i_2})$
and such that $o(x_j)=o(x_{i_1+1})$ for all $i_1+1\leq j\leq i_2$.
Continuing in this way, there are $i_1<i_2<\cdots$ such that $o(x_i)$ shrinks
for all $i=i_j$ and $o(x_i)$ stays constant everywhere else. Thus, if $b$ were negative,
there would be some $i$ such that $o(x_i)<o(b)$, which, since $x_i$ is also negative,
would imply $x_i>b$, contradicting that $x_i\ll b$. This proves Claim 1.

Claim 2: Whenever $a\ll' b$ and $a\geq 0$, then $o(a)>o(b)$.
To see this, assume $a\geq 0$ and $a\ll' b$, so there is a sequence
$a=x_0\ll x_1\ll\cdots$ with each $x_i\ll b$.

Case 1: All the $o(x_i)$ are equal. Then each $LC(x_{i+1})\geq LC(x_i)+r$, so by Lemma
\ref{generalizedarchimedeanlemma}, there is some $i$ such that $LC(x_i)>LC(b)$.
Since $x_i\ll b$, this implies $o(x_i)>o(b)$. Thus $o(a)=o(x_0)=o(x_i)>o(b)$, as desired.

Case 2: There is some minimal $i$ such that $o(x_{i+1})\not=o(x_{i})$.
Since $a\geq 0$ and $a=x_0$, it follows that $x_i\geq 0$ and $x_{i+1}\geq 0$. Since
$x_{i+1}>x_{i}$, this implies $o(x_{i+1})<o(x_{i})$.
Since $x_{i+1}\ll b$, this implies $o(b)\geq o(x_{i+1})>o(x_i)$.
By minimality of $i$, $o(x_i)=o(x_0)=o(a)$, so $o(a)>o(b)$, as desired. This proves Claim 2.

Finally we prove the theorem. Suppose $x_0\ll' x_1\ll' \cdots$, and let
$y$ be any formal Laurent series, we must show there is some $i$ such that
$y\ll x_i$. By Claim 1, $x_1\geq 0$.
Since $x_1\geq 0$ and $x_1\ll' x_2\ll'\cdots$, it follows that $x_i>0$ for all $i\geq 2$.
By Claim 2, for all $i\geq 1$, $o(x_{i+1})<o(x_i)$.
It follows that there is some $i\geq 2$
such that $o(x_i)<o(y)$.
Thus we are in the case where $o(y)>o(x_i)$ and (since $x_i> 0$) $LC(x_i)>0$,
so by Definition \ref{significantorderednessoflaurent}, $y\ll x_i$.
\end{proof}

\subsubsection{Hyperreal numbers}

The field of mathematics where the calculus is formalized with infinite and infinitesimal
quantities is called \emph{nonstandard analysis} \citep{robinson}. The numbers most
commonly associated with this field are the so-called \emph{hyperreal numbers}.

The hyperreal numbers can be
introduced axiomatically or by means of a semi-constructive method which depends on
usage of a certain black box, a device known as a \emph{free ultrafilter}. Logicians
have proven that free ultrafilters exist but that, unfortunately, it is impossible to
concretely exhibit one. This severely limits (if not completely ruins) the practical
usefulness of reinforcement learning with hyperreal rewards.

Nevertheless, the hyperreals
might be useful for proving abstract structural properties about AGI\footnote{Similar to
the way we use free ultrafilters in \citep{alexander2019intelligence} to obtain
comparators of the utility-maximizing ability of traditional deterministic RL agents,
and prove structural properties about said comparators.
In fact, in that paper, we essentially independently re-invented the free ultrafilter
construction of the hyperreals, without realizing it at the time!}.
It can be shown that the hyperreals are not even semi-Archimedean (much less Archimedean).
Thus, for the purpose of
proving abstract theorems about RL agents with fully generalized rewards, the hyperreals
would be more appropriate than the formal Laurent series.

\subsubsection{Surreal numbers}

All of the
well-known non-Archimedean extensions of $\mathbb R$
(including formal Laurent series and hyperreals) are subsystems of the
so-called \emph{surreal numbers}
\citep{conway, knuth, ehrlich2012absolute}. The surreal
numbers were initially discovered during John Conway's attempts to study
two-player combinatorial games like Go and Chess, so it would not be
surprising if they turn out to be important in the eventual development of
AGI.

Unlike the hyperreals,
the construction of the surreal numbers does not depend on any
non-constructive black boxes such as free ultrafilters.
They are constructed as the union of a hierarchy $S_\alpha$ of subsystems where
$\alpha$ ranges over the ordinal numbers. Assuming that agents with AGI are
implemented using computers with no additional power beyond the Church-Turing
Thesis, then for the purposes of AGI, it would be appropriate to restrict our
attention to some computable subset of the surreal numbers, which would
presumably be the union of
some hierarchy $C_\alpha$ where $\alpha$ ranges over the computable ordinal numbers.
For any particular level $C_\alpha$ in this hierarchy, we could consider the
sub-universe $E_\alpha$ of surreal-reward RL environments with rewards restricted
to $C_\alpha$.

Assuming AGI agents are Turing computable,
no individual AGI can possibly comprehend codes for all computable ordinals, because
the set of codes of computable ordinals is badly non-computably-enumerable.
This is profound, because it seems to suggest that any particular AGI can
only comprehend RL environments in $E_\alpha$ if that AGI can comprehend $\alpha$.
In other words, for any particular RL environment $e$ with computable surreal number
rewards, there must be some minimal computable ordinal $\alpha$ such that $e$ has
rewards from $E_\alpha$; if an AGI is not intelligent enough to comprehend $\alpha$,
then it seems like there should be no way for the AGI to comprehend $e$
either\footnote{This situation is reminiscent of \citep{hibbard2011measuring}.}.
We would submit
this state of affairs as evidence in favor of our thesis
\citep{alexander2019measuring} that a machine's intelligence ought to be measured
in terms of the computable ordinals which the machine comprehends.

The above paragraph points at a possible joint path
toward AGI incorporating both machine learning and symbolic logic---toward
``the integration of Symbolic and Statistical AI'' \citep{maruyama}---perhaps
a much-needed reconciliation of these two approaches.

\subsection{Alternate number systems: tentative verdict}

For many simple environments not too far outside of traditional RL,
formal Laurent series could probably serve as a fairly practical number system.
But formal Laurent series have limitations suggesting that RL
with formal Laurent series rewards will probably not be enough to reach AGI, for the
same reason that RL with real number rewards will probably not be enough.

Because of their dependence on free ultrafilters, the hyperreal numbers will
probably never be of practical use as RL rewards, but it could conceivably be
possible to use them to prove abstract structural results about AGI from a
bird's-eye view.

The surreal numbers (or a computable subset thereof) seem like the most promising
candidate for RL rewards that could plausibly lead to AGI. We would certainly
hesitate to call them ``practical'', though. To work with any but the most trivial
of surreal numbers, one would need to implement sophisticated symbolic-logical
machinery, and that's just to get one's foot in the door. This does, however, offer
a ray of hope in that doing deep learning techniques with surreal numbers
could be a way to combine both symbolic logic and statistical methods into a joint
approach.


\section{Conclusion}
\label{conclusionsection}

In traditional reinforcement learning, utility-maximizing agents interact
with environments, receiving real (or rational) number rewards in response to
actions, and using those rewards to update their behavior. We have argued that
the decision to limit rewards to real numbers is inappropriate in the context
of AGI, because the real numbers have the Archimedean property, which makes it
impossible to use them to accurately portray the value of actions when a task
involves inherently non-Archimedean rewards. Thus, we argue, traditional
RL probably will not lead to AGI, because a genuine AGI should have no trouble
comprehending and at least attempting tasks that inherently involve
non-Archimedean rewards. We suggested two possible ways
to modify traditional reinforcement learning to fix this bug: switch to
preference-based reinforcement learning, or else generalize reinforcement learning
to allow rewards from a non-Archimedean number system.

\section*{Acknowledgments}

We gratefully acknowledge Bryan Dawson, Jos{\'e} Hern{\'a}ndez-Orallo,
Mikhail Katz,
Brendon Miller-Boldt, Stewart Shapiro, the
SEC's Quantitative Analytics Unit's machine learning seminar,
and the reviewers
for comments and feedback.

\bibliography{arch}
\end{document}